\documentclass[lettersize,journal]{IEEEtran}

\IEEEoverridecommandlockouts          
\usepackage{mathtools}
\usepackage{gensymb}
\usepackage[export]{adjustbox}
\usepackage{color}
\usepackage{graphicx}
\usepackage{booktabs}
\usepackage{epsfig} 
\usepackage{times} 
\usepackage{amsmath} 
\usepackage{amssymb}
\usepackage{amsthm}
\usepackage{leftidx}
\usepackage{stfloats}
\usepackage[noadjust]{cite}
\newtheorem{assumption}{Assumption}
\newtheorem{property}{Property}
\newtheorem{proposition}{Proposition}

\newtheorem{theorem}{Theorem}[section]
\newtheorem{corollary}{Corollary}[theorem]
\newtheorem{lemma}[theorem]{Lemma}
\newtheorem{remark}{Remark}

\usepackage{siunitx}
\usepackage{nomencl}
\makenomenclature
\usepackage{hyperref}
\usepackage{csquotes}

\usepackage{enumitem}
\usepackage{tablefootnote}
\usepackage{ushort}

\usepackage{algorithm}
\usepackage{algorithmic}

\usepackage{stackengine}
\usepackage{multirow}  

\usepackage[caption=false,font=footnotesize]{subfig}

\title{\LARGE \bf
Physics-Aware Sparse Learning and Selective Online Adaptation for Euler--Lagrange Robot Dynamics
}
\author{Rishabh Dev Yadav$^{1,3}$, Samaksh Ujjawal$^{2}$, Sihao Sun$^{3}$,  Spandan Roy$^{2}$, Wei Pan$^{4}$
\thanks{$^{1}$ Department of Computer Science, The University of Manchaster, U.K. ({\tt \footnotesize  rishabh.yadav@postgrad.manchaster.ac.uk) } }
\thanks{$^{2}$  International Institute of Information Technology Hyderabad. {\tt \footnotesize (samaksh.ujjawal@research.iiit.ac.in, \ spandan.roy@iiit.ac.in}).  }
\thanks{$^{3}$  Department of Cognitive Robotics, Delft University of Technology,
Netherlands  {\tt \footnotesize (s.sun-2@tudelft.nl}).  }
\thanks{$^{4}$  Newcastle University, UK. {\tt \footnotesize (wei.pan2@newcastle.ac.uk}).  }
}

\begin{document}

\maketitle
\thispagestyle{empty}
\pagestyle{empty}

\begin{abstract}
Accurate dynamics models are essential for model-based robotic control, yet nominal Euler--Lagrange models often become inaccurate in the presence of payload variation, unmodeled coupling, friction, aerodynamic effects, and changing operating conditions. Most learning-based correction methods improve prediction accuracy by introducing a single additive residual, but do not preserve the internal mechanical structure of Euler--Lagrange systems. This leads to models that do not preserve symmetry, positive-definiteness, or the coupling between inertia and velocity-dependent terms, which can result in physically inconsistent predictions and reduced reliability when embedded in model-based controllers. We propose a structure-preserving residual learning framework that decomposes model mismatch into an inertia correction, the corresponding induced Coriolis term, and a generalized-force residual. The mechanical component is learned under physical constraints, while the disturbance-sensitive component is represented through a sparse history-dependent latent interaction model and adapted online using Bayesian linear regression. This separation preserves key mechanical structure while restricting adaptation to the part of the dynamics most affected by changing conditions. Experiments across multiple robotic platforms, including mobile, aerial, and manipulator systems, show that the proposed method improves dynamics prediction and trajectory tracking under coupled and time-varying dynamics. These results highlight the value of combining structured residual modeling, compact latent interaction selection, and selective online adaptation for real-world model-based control.
\end{abstract}

\section{Introduction}
Accurate dynamics, typically based on Euler--Lagrange (EL) models, is fundamental to model-based robotic control. However, these first-principles models are rarely precise in practice due to various unmodeled effects such as payload variation, intra-system coupling, friction, aerodynamic effects and disturbances. Learning-based dynamics correction has therefore become an effective way to improve nominal models without discarding their analytical backbone \cite{duong21hamiltonian,saviolo2023learning, lee2024robot,jin2025learning}. Despite strong empirical results, most existing methods are designed primarily to reduce prediction error rather than to preserve the internal EL mechanical structure of the system; even when using temporal context, physical priors, or online adaptation, they typically represent model mismatch as a single additive residual term, thereby collapsing heterogeneous unmodeled effects into an undifferentiated component \cite{scampicchio2025gaussian,miao2025residual,saviolo2022physics,mei2025fast, wang2023neural}.

This is a nontrivial limitation for EL systems, where mismatch does not arise as an arbitrary unstructured correction but through perturbations of mechanically coupled terms such as inertia, Coriolis/centrifugal effects, and potential forces. Treating these effects as a generic state--input map or sequence predictor may improve local predictive accuracy, but it can violate the structural coupling, symmetry, and energy-related properties that underlie physically consistent modeling and principled control design \cite{orsag2018aerial,spong2008robot}. As a result, such models are harder to interpret, less reliable under extrapolation, and more difficult to analyze when embedded in model-based controllers. These limitations motivate residual learning methods that preserve the structured mechanical form of EL mismatch rather than absorbing all discrepancies into a black-box correction \cite{lutter2019deep,gupta2020structured}.

As illustrated in Fig.~\ref{fig:coupling}, real-world systems exhibit significant and time-varying coupling, while many state interactions remain consistently weak, suggesting that residual dynamics are structured and partially sparse.
Moreover, preserving structure alone is insufficient for real robotic deployment, where dominant model mismatch varies with operating conditions. This makes online adaptation necessary, especially when the learned model is embedded within a model-based controller that must remain accurate under nonstationary dynamics \cite{kaushik2020fast,sun2021online,jiahao2023online}. At the same time, residual dynamics are often governed by only a limited number of active interaction patterns, making a parsimonious representation desirable even before adaptation is considered \cite{liu2023model, brunton2025machine}. However, adapting large learned models online is often computationally expensive and data-inefficient, motivating practical update mechanisms that are local, sparse, or otherwise compact \cite{wilcox2020solar,watson2022learning}. More broadly, parsimonious dynamics representations suggest that compact models can improve efficiency and control compatibility without requiring uniformly dense parameterizations \cite{liu2023model,purnomo2023sparse}. These considerations point to a more targeted objective: not generic online adaptation, but selective adaptation that updates only the disturbance-sensitive component while preserving its mechanically structured component.

To address these challenges, we propose a structure-preserving residual learning framework for Euler--Lagrange robotic systems that combines mechanically consistent residual decomposition, parsimonious history-dependent interaction modeling, and selective online adaptation. The proposed model separates residual mismatch into a structured mechanical component and a disturbance-sensitive generalized-force component, so that key Euler--Lagrange structure is preserved while adaptation remains compact and responsive under changing operating conditions. This design yields a residual model that is physically grounded where analytical structure matters and flexibly adaptive where real-world uncertainty is most pronounced. The main contributions of this work are summarized as follows:

\begin{enumerate}
\item \textbf{Structure-preserving residual decomposition:} We formulate residual learning for Euler--Lagrange systems through a decomposition into an inertia correction, the corresponding induced Coriolis/velocity-dependent term, and a generalized-force residual, thereby preserving key mechanical structure instead of learning a single undifferentiated residual.
\item \textbf{Parsimonious history-dependent interaction modeling:} We introduce a sparse latent representation of recent system history that captures a limited number of active interaction patterns, enabling compact modeling of both structured mechanical mismatch and disturbance-related effects.
\item \textbf{Selective Bayesian online adaptation:} We develop an online adaptation mechanism that updates only the disturbance-sensitive generalized-force component within the active latent subspace, while keeping the mechanically structured component fixed.
\end{enumerate}
We demonstrate improved predictive accuracy and trajectory tracking across multiple robotic platforms operating under coupled and changing dynamics.

\begin{figure*}[]
\centering
\vspace{-2mm}
\includegraphics[width=0.99\linewidth]{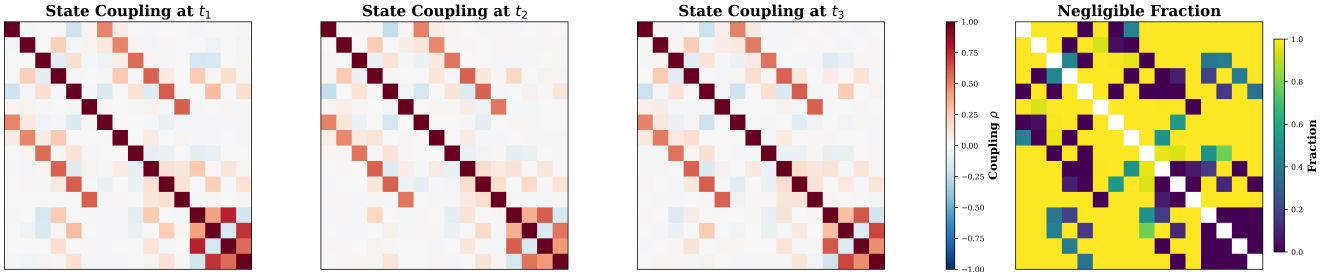}
    \caption{
    \textbf{Time-varying state coupling in an aerial manipulator system.}
    Three panels from left side show the \emph{instantaneous state coupling} matrices at three representative time instants $t_1$, $t_2$, and $t_3$. 
    The coupling matrix is computed using the Pearson correlation coefficient between state trajectories.
    Rightmost panel shows the \emph{negligible coupling fraction}, defined as $
        \eta_{ij} = \frac{1}{T} \sum_{k=1}^{T} 
        \mathbf{1} \left( |\rho_{ij}(t_k)| < \epsilon \right)$,
    where $\mathbf{1}(\cdot)$ is the indicator function and $\epsilon$ is a small threshold (here $\epsilon = 0.08$). 
    This metric quantifies how frequently a pair of states remains weakly coupled over time.
    The results reveal three key properties: 
    (i) significant off-diagonal structure confirms strong coupling between platform and manipulator dynamics, 
    (ii) coupling patterns vary over time due to changing configurations and dynamic interactions, and 
    (iii) a subset of state pairs remains consistently negligible, indicating potential sparsity that can be exploited for structured modeling and control.
    } \label{fig:coupling}
\end{figure*}

\section{Related Work}
Additive residual learning approaches \cite{bauersfeld2021neurobem, cao2024computation, shi2019neural, djeumou2023learn, kim2025modular, das2025dronediffusion}, as well as full-model learning methods \cite{saviolo2022physics, giacomuzzo2024black}, are typically trained offline on fixed datasets without explicitly decomposing Euler--Lagrange dynamics. As a result, they are poorly suited to time-varying operating conditions, where distribution shift can accumulate into prediction error and degrade closed-loop control performance in the absence of online adaptation \cite{lambert2022investigating}.

Structure-aware neural dynamics models such as \cite{lutter2019deep, gupta2020structured, gu2024physics, liu2024physics} learn full system models by embedding Lagrangian, Hamiltonian, or manipulator structure directly into the learned dynamics, improving physical plausibility and sample efficiency. However, these methods are generally trained as fixed models and do not primarily address online adaptation under changing conditions. Among them, \cite{lutter2019deep} is notable in enabling online-capable Euler--Lagrange-structured learning, but its formulation does not explicitly model non-conservative effects such as friction, drag, or actuator-induced disturbances. Physically consistent identification methods \cite{lutter2023differentiable, khorshidi2025physically} instead infer global physical parameters, such as mass, inertia, and friction, which limits their ability to handle payload changes, contact variations, or external disturbances that evolve during deployment.

To address nonstationary dynamics, several methods update learned models during deployment. Some approaches adapt only the final layer of a neural dynamics model using online gradient-based updates \cite{saviolo2023active}, but this requires careful tuning of learning rate and batch size to remain stable and responsive under noisy data. Others rely on meta-learned low-dimensional adaptation, such as basis-weight updates or embedding selection \cite{o2022neural,kaushik2020fast}, which improves data efficiency but ties online adaptation to the expressiveness of a pre-trained family of representations. Nonparametric and system-identification approaches, including sparse/local GP models and online differentiable simulation \cite{wilcox2020solar,watson2022learning,chen2022real,sun2021online}, provide effective online refinement but typically adapt generic local models or global parameter blocks rather than separating fast-changing disturbance effects from the mechanically structured part of the dynamics. In contrast, our approach updates only the generalized-force component, while keeping the inertia and Coriolis corrections fixed.

Beyond structure and adaptivity, compactness also matters when learned dynamics are used inside control loops. Sparse modeling has therefore been used in system identification to select relevant dynamics terms and reduce overparameterization \cite{zhou2022sparse}, as well as to recover parsimonious governing equations or Lagrangian descriptions from data \cite{brunton2016discovering,purnomo2023sparse,zhang2026differentiable}. In control-oriented learning, sparsification has also been used to simplify learned models so that prediction and optimization remain computationally tractable \cite{liu2023model,brunton2025machine}. These works highlight the value of compact representations, but primarily target model discovery, parameter selection, or model simplification. In contrast, we use sparsity to represent residual dynamics through a small number of active interaction patterns within a structure-preserving formulation, enabling a compact model for efficient and selective online adaptation.

The above discussion highlights a gap between structure preservation, online adaptation, and compact modeling. Existing methods either learn unstructured residuals, enforce structure without selective adaptation, or adapt models without separating mechanical and disturbance effects. We address this by proposing a structure-preserving residual framework for Euler--Lagrange systems that separates these components, enabling selective adaptation of the generalized-force term while maintaining a compact and physically consistent model.

\section{Methodology} \label{sec:methodology}

\subsection{Problem Formulation}

We consider a mechanical system described by Euler--Lagrange (EL) dynamics as 
\begin{equation}
M(\chi)\ddot{\chi}
+
C(\chi,\dot{\chi})\dot{\chi}
+
f(\chi,\dot{\chi})
=
\tau,
\label{eq:EL_general}
\end{equation}
where \(
\chi \in \mathbb{R}^{n}
\)
denote the generalized coordinates, \(M(\chi) \in \mathbb{R}^{n \times n}\) is the inertia matrix, \(C(\chi,\dot{\chi}) \in \mathbb{R}^{n \times n}\) is the Coriolis/centrifugal matrix, \(f(\chi,\dot{\chi}) \in \mathbb{R}^{n}\) collects gravity and non-conservative forces, and  \(\tau \in \mathbb{R}^{n}\)  is the generalized input.
Standard properties for Euler–Lagrange systems hold~\cite{spong2008robot}:
\begin{property}
\label{prop_1}
The inertia matrix is uniformly positive definite $\forall \chi$, i.e.,
$ \underline{m}I \preceq M(\chi) \preceq \overline{m}I $, 
for some constants $0 < \underline{m} \le \overline{m}$.
\end{property}

\begin{property}
\label{prop_2}
The matrix $\dot{M}(\chi) - 2C(\chi,\dot{\chi})$ is skew symmetric.
\end{property}

A nominal model \(\bar M(\chi), \bar C(\chi,\dot{\chi}), \bar f(\chi,\dot{\chi})\) is assumed available, typically derived from first principles, where $\bar M(\chi)$ is assumed to be symmetric positive definite. Due to unmodeled dynamics, parameter uncertainty, and environment-dependent effects, the true system deviates from this nominal description. We therefore write
\begin{equation}
\big(\bar M + \Delta M\big)\ddot{\chi}
+
\big(\bar C + \Delta C\big)\dot{\chi}
+
\bar f + \Delta f
=
\tau,
\label{eq:EL_residual}
\end{equation}
where \(\Delta M(\cdot)\), \(\Delta C(\cdot)\), and \(\Delta f(\cdot)\) denote unknown corrections.
The notation $(.)_t$ indicates the value of a variable at time $t$. At time step \(t\), defining the target residual dynamics 
\begin{equation}
\delta_t = \tau_t - \big( \bar M(\chi_t)\ddot{\chi}_t + \bar C(\chi_t,\dot{\chi}_t)\dot{\chi}_t +
\bar f(\chi_t,\dot{\chi}_t)
\big),
\label{eq:residual_target_general}
\end{equation}
we obtain residual model
\begin{equation}
\delta_t = \Delta M_t \ddot{\chi}_t + \Delta C_t \dot{\chi}_t + \Delta f_t.
\label{eq:structured_residual_general}
\end{equation}

The objective is to learn a structured residual model \eqref{eq:structured_residual_general} that preserves the mechanical properties of Euler--Lagrange systems while enabling adaptation to time-varying effects. In particular, the corrected inertia must remain symmetric positive definite, velocity-dependent coupling must remain consistent with the inertia rather than being modeled independently, and the generalized force component must admit efficient online adaptation. To achieve this, we construct a history-dependent latent representation, enforce sparse activation of interaction modes, parameterize structured corrections to inertia, and restrict online adaptation to the generalized force component.

\subsection{Temporal History Encoder}

The residual dynamics in~\eqref{eq:structured_residual_general} depends not only on the instantaneous state but also on recent system evolution. To capture these short-horizon dependencies, we encode a finite history of observations.
At time step \(t\), define the per-step observation
\(
\zeta_t
=
\begin{bmatrix}
\chi_t^\top &
\dot{\chi}_t^\top &
\ddot{\chi}_t^\top &
\tau_t^\top
\end{bmatrix}^\top
\in \mathbb{R}^{4n},
\)
and construct a history window of length \(K\):
\begin{equation}
H_t = \big[\zeta_{t-K+1}, \dots, \zeta_t\big]^\top \in \mathbb{R}^{K \times 4n}.
\label{eq:history_tensor_general}
\end{equation}

The history tensor \(H_t\) is mapped to a latent representation using a temporal encoder
\begin{equation}
\tilde z_t = \phi(H_t) \in \mathbb{R}^{d_z},
\label{eq:dense_latent_general}
\end{equation}
where \(\phi(\cdot)\) is implemented as a lightweight temporal convolutional network operating along the time dimension.

The latent feature \(\tilde z_t\) provides a compact, history-dependent representation of system evolution. It is subsequently used to extract sparse latent interaction modes for structured residual modeling.

\subsection{Sparse Latent Interaction Selection}

The latent representation \(\tilde z_t\) encodes a dense set of history-dependent features, but residual dynamics are typically governed by a small subset of active interaction modes at any given time. To exploit this structure, we extract sparse latent representations that selectively activate 
a subset of latent interaction components. Each latent dimension corresponds to a learned interaction component capturing a distinct pattern of residual dynamics (e.g., configuration-dependent coupling, velocity-dependent effects, or disturbance signatures).

We define two branch-specific sparse codes from \(\tilde z_t\):
\begin{equation}
z_t^{M} = \mathcal{S}_{\lambda_M}(\tilde z_t),
\qquad
z_t^{f} = \mathcal{S}_{\lambda_f}(\tilde z_t),
\label{eq:sparse_branches_general}
\end{equation}
where \(z_t^{M}, z_t^{f} \in \mathbb{R}^{d_z}\), and \(\mathcal{S}_{\lambda}(\cdot)\) is the elementwise soft-thresholding operator
\begin{equation}
\big[\mathcal{S}_{\lambda}(v)\big]_i
=
\mathrm{sign}(v_i)\max\{|v_i|-\lambda,\,0\}, 
\quad i=1,\dots,d_z.
\label{eq:soft_threshold_general}
\end{equation}
The soft-thresholding operator is non-differentiable at zero; in practice, training is performed using subgradient methods, which are standard in sparse representation learning.

The structural branch \(z_t^{M}\) parameterizes corrections to the inertia matrix, while the adaptive branch \(z_t^{f}\) drives the generalized force model. 
Sparsity is activation-dependent: the support of \(z_t^{M}\) and \(z_t^{f}\) varies with the input history. For the adaptive branch, we define the active index set
\begin{equation}
S_t = \mathrm{supp}(z_t^{f}) = \{\, i \mid (z_t^{f})_i \neq 0 \,\},
\label{eq:active_support_general}
\end{equation}
i.e., the set of active latent components at time $t$. Here, $\mathrm{supp}(v)$ denotes the support of a vector $v$, defined as the index set of its nonzero entries. Only the coordinates in $S_t$ contribute to prediction and are updated during online adaptation.

This construction yields a compact representation of residual dynamics in terms of a small number of active latent modes. The structural branch promotes parsimonious parameterization of inertia corrections, while the adaptive branch defines a low-dimensional, time-varying subspace for efficient online updates.

\subsection{Positive-Definite Parameterization of the Inertia Matrix}

For Euler--Lagrange systems, the inertia matrix must remain symmetric positive definite (SPD) for all admissible configurations. We enforce this constraint by parameterizing the correction as a structured positive-definite matrix composed of a low-rank term and a strictly positive diagonal component.

The corrected inertia matrix is defined as
\begin{equation}
\hat M(\chi_t, z_t^M)
=
\bar M(\chi_t)
+
\Delta M_{\theta_M}(\chi_t, z_t^M),
\label{eq:corrected_inertia_general}
\end{equation}
where \(\bar M(\chi_t) \in \mathbb{R}^{n \times n}\) is the nominal inertia matrix (assumed SPD), and \(\Delta M_{\theta_M}(\cdot)\) is a learned correction driven by the sparse structural latent code \(z_t^M \in \mathbb{R}^{d_z}\).

\paragraph{Low-rank and diagonal parameterization.}
We use a single neural decoder
\begin{equation}
\big[b(\chi_t, z_t^M),\; d(\chi_t, z_t^M)\big]
=
g_{\theta_M}(\chi_t, z_t^M),
\end{equation}
where \(g_{\theta_M}(\cdot)\) is a multilayer perceptron with output dimension \(nr + n\). The vector \(b(\cdot) \in \mathbb{R}^{nr}\) is reshaped as
\begin{equation}
B(\chi_t, z_t^M)
=
\mathrm{mat}\big(b(\chi_t, z_t^M)\big)
\in \mathbb{R}^{n \times r},
\end{equation}
and \(d(\cdot) \in \mathbb{R}^{n}\) parameterizes a diagonal matrix.

The inertia correction is defined as
\begin{equation}
\Delta M_{\theta_M}(\chi_t, z_t^M)
=
B(\chi_t, z_t^M)\,B(\chi_t, z_t^M)^\top
+
D(\chi_t, z_t^M),
\label{eq:learned_deltaM}
\end{equation}
where $
D(\chi_t, z_t^M)
=
\mathrm{diag}\!\big(
\mathrm{softplus}(d(\chi_t, z_t^M))
\big)$.

\paragraph{Positive-definiteness.}
The term \(B B^\top\) is symmetric positive semidefinite, and \(D\) is diagonal with strictly positive entries. Therefore,
\begin{equation}
\Delta M_{\theta_M}(\chi_t, z_t^M) \succ 0,
\quad
\hat M(\chi_t, z_t^M) \succ 0 \;\; \forall (\chi_t, z_t^M).
\end{equation}

This parameterization guarantees symmetry and positive definiteness by construction. The low-rank factor \(B \in \mathbb{R}^{n \times r}\) captures coupling effects across generalized coordinates, while the diagonal term ensures numerical conditioning and models independent inertial scaling. Since the correction is driven by the sparse latent code \(z_t^M\), only a subset of interaction modes is active at each time step, yielding a parsimonious and well-structured representation of inertial variations.

\subsection{Structure-Preserving Coriolis Matrix Construction}
\label{sec:coriolis_structure_general}

In Euler--Lagrange systems, the Coriolis and centrifugal terms are determined by the inertia matrix and are not independent. To preserve this coupling, we construct the Coriolis matrix directly from the corrected inertia \(\hat M(\chi_t, z_t^M)\) in~\eqref{eq:corrected_inertia_general} rather than learning a separate residual.
 The Christoffel symbols of the first kind are given by
\begin{equation}
\Gamma_{ijk}(\chi_t, z_t^M)
=
\frac{1}{2}
\left(
\frac{\partial \hat M_{ij}}{\partial \chi_k}
+
\frac{\partial \hat M_{ik}}{\partial \chi_j}
-
\frac{\partial\hat  M_{jk}}{\partial \chi_i}
\right),
\label{eq:christoffel_symbols_general}
\end{equation}
for \(i,j,k = 1,\dots,n\). The Coriolis matrix is then defined as
\begin{equation}
\hat C_{ij}(\chi_t,\dot{\chi}_t, z_t^M)
=
\sum_{k=1}^{n}
\Gamma_{ijk}(\chi_t, z_t^M)\,\dot{\chi}_{t,k}.
\label{eq:coriolis_from_christoffel_general}
\end{equation}

This construction enforces consistency between inertial and velocity-dependent terms by deriving the Coriolis matrix directly from the learned inertia. The resulting correction is
\begin{equation}
\Delta C_{\theta_M}(\chi_t,\dot{\chi}_t, z_t^M)
=
\hat C(\chi_t,\dot{\chi}_t, z_t^M)
-
\bar C(\chi_t,\dot{\chi}_t),
\label{eq:delta_c_general}
\end{equation}
which introduces no additional parameters. Derivatives of $\hat M(\chi_t, z_t^M)$ are computed via automatic differentiation, yielding approximate preservation of the $\dot M - 2C$ skew-symmetry.

\subsection{Generalized Force Residual Modeling}

After accounting for inertial and its induced Coriolis structure, the remaining mismatch is attributed to generalized forces. These terms capture non-conservative and unmodeled effects that do not modify the kinetic energy structure.
We model the generalized force correction as 
\begin{equation} \label{eq:corrected_f}
 \hat f(\chi_t,\dot{\chi}_t,z_t^f) = \bar f(\chi_t,\dot{\chi}_t) + \Delta f_{\theta_f}(z_t^f)   
\end{equation}
using the adaptive latent code \(z_t^f\)  where
\begin{equation}
\Delta f_{\theta_f}(z_t^f)
=
\Theta z_t^f,
\qquad
\Theta \in \mathbb{R}^{n \times d_z},
\label{eq:force_decoder_general}
\end{equation}
where \(\Theta\) is the force decoder.
The use of a linear decoder follows from the role of the latent representation: the encoder \(\phi(H_t)\) captures nonlinear history-dependent interactions, while \(\Theta\) maps the resulting latent modes to generalized forces.

Combining~\eqref{eq:learned_deltaM}, \eqref{eq:delta_c_general}, and~\eqref{eq:force_decoder_general}, the predicted residual dynamics are written as
\begin{equation}
\hat{\delta}_t
=
\Delta M_{\theta_M}(\chi_t, z_t^M)\ddot{\chi}_t
+
\Delta C_{\theta_M}(\chi_t,\dot{\chi}_t, z_t^M)\dot{\chi}_t
+
\Delta f_{\theta_f}(z_t^f).
\label{eq:structured_model_general}
\end{equation}

This formulation separates residual dynamics into structured corrections driven by \(z_t^M\) and disturbance-like effects captured by \(z_t^f\), enabling selective adaptation without modifying the learned mechanical structure.

\subsection{Training Objective}

The model is trained from a dataset 
\(
\mathcal{D} = \{(H_t, \delta_t)\}_{t=1}^{N},
\)
where \(H_t\) is the history window defined in~\eqref{eq:history_tensor_general} and \(\delta_t\) is the residual target in~\eqref{eq:residual_target_general} and $\hat{\delta}_t$ predicted residual from \eqref{eq:structured_model_general}.

\paragraph{Data fidelity.}
The primary objective minimizes the residual prediction error:
\begin{equation}
\mathcal{L}_{\mathrm{data}}
=
\frac{1}{N}
\sum_{t=1}^{N}
\left\|
\delta_t - \hat{\delta}_t
\right\|_2^2.
\end{equation}

\paragraph{Sparsity regularization.}
To promote consistent activation of a small number of latent modes, we include an \(\ell_1\) penalty on the sparse codes, with $\lambda_M,\lambda_f \ge 0$:
\begin{equation}
\mathcal{L}_{\mathrm{sparse}}
=
\frac{1}{N}
\sum_{t=1}^{N}
\left(
\lambda_M \|z_t^M\|_1
+
\lambda_f \|z_t^f\|_1
\right).
\end{equation}

\paragraph{Total objective.}
\begin{equation} \label{eq:final_loss}
\mathcal{L}
=
\mathcal{L}_{\mathrm{data}}
+
\mathcal{L}_{\mathrm{sparse}}.
\end{equation}

\section{Online Adaptation} \label{sec:online_update}

To account for time-varying effects, we adapt only the generalized force decoder while keeping the inertia and Coriolis components fixed. This preserves the learned mechanical structure and confines adaptation to disturbance-sensitive terms.

\paragraph{Residual isolation.}
Using the residual target \(\delta_t\) defined in~\eqref{eq:residual_target_general}, the force residual is obtained as
\begin{equation} \label{eq:residual_isolation}
r_t
=
\delta_t
-
\Delta M_{\theta_M}(\chi_t, z_t^M)\ddot{\chi}_t
-
\Delta C_{\theta_M}(\chi_t,\dot{\chi}_t, z_t^M)\dot{\chi}_t.
\end{equation}

This yields the linear model
$r_t = \Theta z_t^f + \varepsilon_t$,
where \(\varepsilon_t\) is zero-mean noise.

\paragraph{Bayesian linear regression.}
Each output dimension is modeled independently as $
r_t^{(j)} = (z_t^f)^\top \theta_j + \varepsilon_{t,j}$,
where \(\theta_j \in \mathbb{R}^{d_z}\) is the parameter vector for the \(j\)-th output. The noise satisfies $
\varepsilon_t \sim \mathcal{N}(0, \Sigma_\varepsilon),
~~\Sigma_\varepsilon = \mathrm{diag}(\sigma_1^2,\dots,\sigma_n^2)$ for computational tractability in recursive updates.
A Gaussian posterior is maintained for each \(\theta_j\): $
\theta_j \sim \mathcal{N}(\mu_{t,j}, P_{t,j})$,
initialized from the offline-trained parameters.

\paragraph{Bayesian update.}
Each output dimension $j = 1,\dots,n$ is updated independently, where
$\mu_{t,j} \in \mathbb{R}^{d_z}$ and $P_{t,j} \in \mathbb{R}^{d_z \times d_z}$ denote the posterior mean and covariance, respectively, and $\sigma_j^2 > 0$ is the observation noise variance. Let $z_t^f \in \mathbb{R}^{d_z}$ denote the sparse latent code obtained via soft-thresholding. Updates are performed as
{\small
\begin{align}
K_{t,j} &= P_{t-1,j} z_t^f
\left((z_t^f)^\top P_{t-1,j} z_t^f + \sigma_j^2\right)^{-1}, \\
\mu_{t,j} &= \mu_{t-1,j}
+
K_{t,j}
\left(r_t^{(j)} - (z_t^f)^\top \mu_{t-1,j}\right), \nonumber \\
P_{t,j} &=
\big(I - K_{t,j} (z_t^f)^\top\big)
P_{t-1,j}
\big(I - K_{t,j} (z_t^f)^\top\big)^\top
+
\sigma_j^2 K_{t,j} K_{t,j}^\top. \nonumber
\end{align}
}
The updated decoder is obtained as $
\Theta_t = [\mu_{t,1}^\top \;\cdots\; \mu_{t,n}^\top]^\top \in \mathbb{R}^{n \times d_z}$.
Although updates are performed in the full $d_z$-dimensional space, the sparsity of $z_t^f$ implies that only a subset of latent directions contributes significantly at each time step. This induces an effective low-dimensional adaptation without requiring explicit masking or dynamic resizing of the parameter space, while remaining compatible with fixed-size tensor implementations.

\section{Control Law}

We employ a computed torque controller based on the learned dynamics model. Let \(e = \chi - \chi_d\) and \(\dot e = \dot{\chi} - \dot{\chi}_d\). The control input is defined as
\begin{equation}
\tau
=
\hat M(\chi_t,z_t^M) u
+
\hat C(\chi_t,\dot{\chi}_t,z_t^M)\dot{\chi}
+
\hat f(\chi_t,\dot{\chi}_t, z_t^f), \label{eq:control_law}
\end{equation}
where 
$u = \ddot{\chi}_d - K_d \dot e - K_p e$, with positive definite gains \(K_p, K_d\).
The learned model components are given by \eqref{eq:corrected_inertia_general}, \eqref{eq:coriolis_from_christoffel_general} and \eqref{eq:corrected_f} respectively. The learned dynamics model provides feedforward compensation of the system dynamics, while the feedback gains compensate for trajectory tracking errors due to model inaccuracies.

\begin{remark}
For underactuated systems with input dimension $m < n$, the dynamics are written as $
M(\chi)\ddot{\chi} + C(\chi,\dot{\chi})\dot{\chi} + f(\chi,\dot{\chi}) = B(\chi) \tau_{ua}$,
where $B(\chi) \in \mathbb{R}^{n \times m}$ is assumed known. During data collection, generalized forces are computed as $\tau = B(\chi)\tau_{ua}$. At control time, the desired generalized force $\tau^\star$ is mapped to actuator inputs via the Moore--Penrose pseudoinverse, $\tau_{ua} = B(\chi)^\dagger \tau^\star$,
yielding the closest realizable control input in a least-squares sense.   
\end{remark}

\begin{algorithm}[t]
\caption{Structured Residual Dynamics Learning}
\label{alg:method}

\textbf{Offline Training}
\begin{algorithmic}[1]
\STATE Construct dataset $\mathcal{D} = \{(H_t, \delta_t)\}_{t=1}^{N}$ using~\eqref{eq:history_tensor_general} and~\eqref{eq:residual_target_general}
\FOR{each training sample}
    \STATE Encode history: $\tilde z_t = \phi(H_t)$  \eqref{eq:dense_latent_general}
    \STATE Apply sparse selection: $(z_t^M, z_t^f)$  \eqref{eq:sparse_branches_general}
    \STATE Compute structured prediction $\hat{\delta}_t$  \eqref{eq:structured_model_general}
\ENDFOR
\STATE Optimize parameters using objective~\eqref{eq:final_loss}
\STATE Initialize online decoder parameters $(\mu_{0,j}, P_{0,j})$ for all $j$
\end{algorithmic}

\vspace{0.5em}
\textbf{Online Deployment}
\begin{algorithmic}[1]
\STATE Observe $(\chi_t, \dot{\chi}_t, \ddot{\chi}_t, \tau_t)$ and update history $H_t$
\STATE Encode history: $\tilde z_t = \phi(H_t)$
\STATE Apply sparse selection: $(z_t^M, z_t^f)$
\STATE Compute control input $\tau_t$ using learned model \eqref{eq:control_law}
\STATE Observe next state and compute residual $\delta_t$  \eqref{eq:residual_target_general}
\STATE Compute force residual $r_t$ \eqref{eq:residual_isolation} 
\STATE Extract active set $S_t = \mathrm{supp}(z_t^f)$  \eqref{eq:active_support_general}
\STATE Update decoder parameters $(\mu_{t,j}, P_{t,j})$ sec. \ref{sec:online_update}
\STATE Update $\Theta_t$ from posterior means
\end{algorithmic}
\end{algorithm}

\section{Experiments}

We compared our method against several established baselines. For model prediction performance, we consider sparse system identification (SysID)~\cite{brunton2016discovering}, physics-inspired temporal convolutional networks (PI-TCN)~\cite{saviolo2022physics}, and DMRAC~\cite{joshi2021asynchronous}, an encoder--decoder-based approach. 
For trajectory tracking performance, we compare against deep neural networks (DNN)~\cite{shi2019neural}, diffusion-based models~\cite{das2025dronediffusion}, meta-learning approaches~\cite{o2022neural}, and active MLPs~\cite{saviolo2023active}, which adapt only the final layer during online execution.
All methods are used to learn residual dynamics models and are integrated with the same controller \eqref{eq:control_law} to ensure a fair comparison, isolating the effect of the model learning approach. Additionally, we provide an ablation study to evaluate the role of sparsity in the proposed method. All models are trained and evaluated under identical conditions across 10 independent trials.

\subsection{Hardware Setup}
We evaluate the proposed method on five robotic platforms: a differential-drive mobile robot, a 6-searial DoF manipulator, a quadrotor, a mobile manipulator with a holonomic base with 6-DoF arm and an aerial manipulator with a 2-DoF serial arm. All systems are equipped with an onboard NVIDIA Jetson Orin NX for real-time computation, and the same model architecture and training procedure are used across platforms.
Each system is described by generalized coordinates \(\chi\) and control inputs \(\tau \), with dimensions summarized in Table~\ref{tab:system_dimension}. High-level torque or force commands are generated at \(70\) Hz using the learned model and controller, while onboard low-level controllers convert these commands into actuator inputs.

Nominal dynamics models are assumed known for all platforms. For the mobile manipulator and aerial manipulator, the nominal model is constructed by treating the base and manipulator as decoupled subsystems, thereby neglecting coupling effects. 
State feedback is obtained from onboard sensing and external motion capture and filtered accelerations are computed numerically where sensor feedback not available. All experiments are conducted in closed-loop with real-time inference running onboard.

\begin{table}[t]
\centering
\footnotesize
\caption{System dimensions for different platforms.}
\label{tab:system_dimension}
\begin{tabular}{lccccc}
\toprule
\textbf{Dim.} 
& \textbf{RP1} 
& \textbf{RP2} 
& \textbf{RP3} 
& \textbf{RP4} 
& \textbf{RP5} \\
\midrule
$\chi$ 
& 3 & 6 & 6 & 8 & 9 \\

$\tau$ 
& 2 & 6 & 4 & 6 & 8 \\
\bottomrule
\end{tabular}

\vspace{0.3em}
\raggedright
{\scriptsize $\chi$ and $\tau$ denotes system state and control input. RP1: Mobile Robot, RP2: Manipulator, RP3: Quadrotor, RP4: Mobile Manipulator, RP5: Aerial Manipulator.}
\end{table}

\subsection{Data Collection}
Training data are collected by executing smooth, randomly generated trajectories using a PID controller to excite a wide range of system dynamics. The trajectories vary position, velocity, and actuation profiles over time, and $5$ minutes of data are recorded at 100\,Hz, forming a time-indexed dataset  
$ \left[ \zeta_1, \zeta_2, \dots) \right]^\top,
$ sampled at timestamps \( \{t_1, t_2, \dots\} \). 
Data are collected with zero and two different payload, while evaluation is performed under unseen conditions, including payload variations, novel trajectory and changes in operating regimes. This setup enables the learned model to capture residual effects such as friction, coupling, mass variation, and aerodynamic disturbances.

\subsection{Implementation Details}

The proposed model is implemented in JAX and weights are initialized using Xavier initialization. The same model architecture and training configuration are used across all platforms.
The encoder and structural decoder parameters are optimized jointly, while the generalized force decoder is initialized from offline training. A fixed history length is used to construct input sequences, and training is performed with mini-batch stochastic gradient descent.
During deployment, only the generalized force decoder is updated using Bayesian linear regression.
A computed torque controller is used with gains \(K_p\) and \(K_d\), which are selected independently for each platform to ensure stable tracking performance.  The parameter details for offline, online and controller is provided in Table~\ref{tab:impl_details}, \ref{tab:online_details} and \ref{tab:control_parameter}.

The desired roll and pitch of the quadrotor and aerial manipulator are calculated following (\cite{mellinger2011minimum}). 
For all learned models, the input features consist of the current state $\chi_t$, $\dot{\chi}_t$, $\ddot{\chi}_t$ and the previous control input $\tau_{t-1}$, since $\tau_t$ is not available at prediction time~(\cite{shi2019neural}).

\begin{table}[t]
\centering
\footnotesize
\caption{Model and training configuration (shared across platforms)}
\label{tab:impl_details}
\begin{tabular}{l c}
\toprule
\textbf{Parameter} & \textbf{Value} \\
\midrule

\multicolumn{2}{c}{\textit{Temporal Encoder}} \\
History length ($K$) & 5 \\
Encoder type & Temporal CNN (2 layers) \\
Kernel size & 3 \\
Hidden dimension & 64 \\
Activation & ELU $\in C^1(\mathbb{R})$ \\

\midrule
\multicolumn{2}{c}{\textit{Latent Representation}} \\
Latent dimension ($d_z$) & 16 \\
Sparsity thresholds ($\lambda_M, \lambda_f$) & 0.01, 0.01 \\
Soft-threshold ($\lambda$) & 0.2 \\

\midrule
\multicolumn{2}{c}{\textit{Inertia Decoder}} \\
MLP layers & 2 \\
Hidden dimension & 64 \\
Activation & ELU \\
Low-rank dimension ($r$) & $\left\lceil \frac{3n}{4} \right\rceil$ \\

\midrule
\multicolumn{2}{c}{\textit{Training}} \\
Optimizer & Adam \\
Learning rate & $1 \times 10^{-3}$ \\
Batch size & 256 \\
Training epochs & 100 \\
Weight decay & $1 \times 10^{-5}$ \\

\bottomrule
\end{tabular}
\end{table}

\begin{table}[t]
\centering
\footnotesize
\caption{Online adaptation parameters}
\label{tab:online_details}
\begin{tabular}{l c}
\toprule
\textbf{Parameter} & \textbf{Value} \\
\midrule

Initial mean ($\mu_{0,j}$) & offline trained weights \\
Initial covariance ($P_{0,j}$) & $0.1\,I$ \\
Observation noise ($\sigma_j^2$) & $1 \times 10^{-3}$ \\

\bottomrule
\end{tabular}
\end{table}

\begin{table}[t]
\centering
\footnotesize
\caption{Controller gains for different platforms}
\begin{tabular}{lcc}
\toprule
\textbf{Platform} & $K_p$ & $K_d$ \\
\midrule
Mobile Robot & 4*diag(I) & 2*diag(I) \\
Manipulator & 6*diag(I) & 2*diag(I) \\
Quadrotor & 3*diag(I) & 2*diag(I) \\
Aerial Manipulator & 4*diag(I) & 3*diag(I) \\
Mobile Manipulator & 6*diag(I) & 3*diag(I) \\
\bottomrule
\end{tabular} \label{tab:control_parameter}
\end{table}

\subsection{Experimental Scenarios}
\begin{figure*}[]
\centering
\vspace{-2mm}
\includegraphics[width=0.99\linewidth]{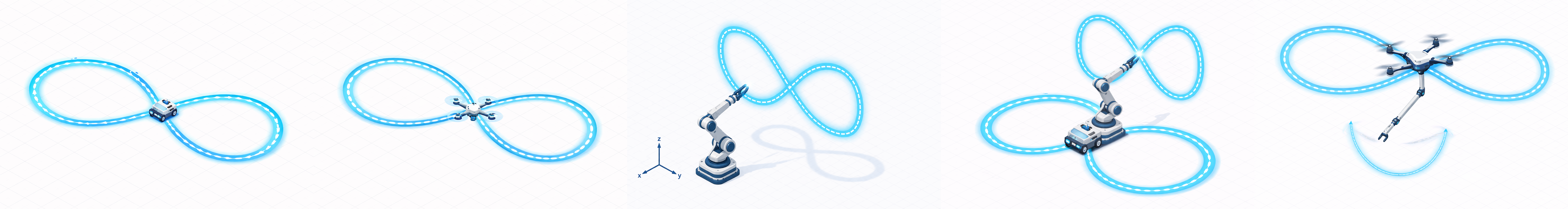}
\caption{Experimental scenarios across robotic platforms. Each system is tasked with tracking a figure-eight (infinity-shaped) trajectory.}
    \label{fig:experiment_scenerio}
\end{figure*}

We evaluate the proposed method on multiple robotic platforms under dynamically rich reference trajectories designed to excite coupled system dynamics. The trajectories for each platform are illustrated in Fig.~\ref{fig:experiment_scenerio}.
The ground mobile robot and quadrotor are commanded to follow a planar figure-eight (infinity-shaped) trajectory in the horizontal ($xy$) plane. For the fixed-base manipulator, the end-effector tracks a figure-eight trajectory in the vertical ($yz$) plane.
For the mobile manipulator, the base follows a planar figure-eight trajectory in the $xy$ plane, while the manipulator end-effector simultaneously tracks a figure-eight trajectory in the $yz$ plane relative to the moving base. For the aerial manipulator, the aerial base follows a figure-eight trajectory in the $xy$ plane, while the attached $2$-DoF serial arm executes periodic oscillatory motion.
These trajectories are chosen to induce sufficiently rich excitation in both configuration and velocity space, enabling meaningful evaluation of model prediction and control performance under coupled and nonlinear dynamics.

\subsection{Model Validation}
For predictive accuracy evaluation, each robot tracks a smooth reference trajectory using a vanilla PID controller for $2$ minutes under operating conditions not seen during training (e.g., payload variation). This setup induces model mismatch and allows assessment of generalization performance.
The prediction error is computed in the residual dynamics space. For baseline methods that directly learn the residual $\hat{\delta}_t$, the error is defined as 
$\delta_e = \tau_t - \big( \bar M\ddot{\chi}_t + \bar C\dot{\chi}_t + \bar f \big) - \hat{\delta}_t$.
For the proposed structured model, the predicted residual is decomposed into inertia, Coriolis, and force components, yielding $ \delta_e = \tau_t - \big( \bar M\ddot{\chi}_t + \bar C\dot{\chi}_t + \bar f \big) - \Delta M_{\theta_M} - \Delta C_{\theta_M} - \Delta f_{\theta_f}$.

\begin{figure}[!h]
\centering
\includegraphics[width=0.98\linewidth]{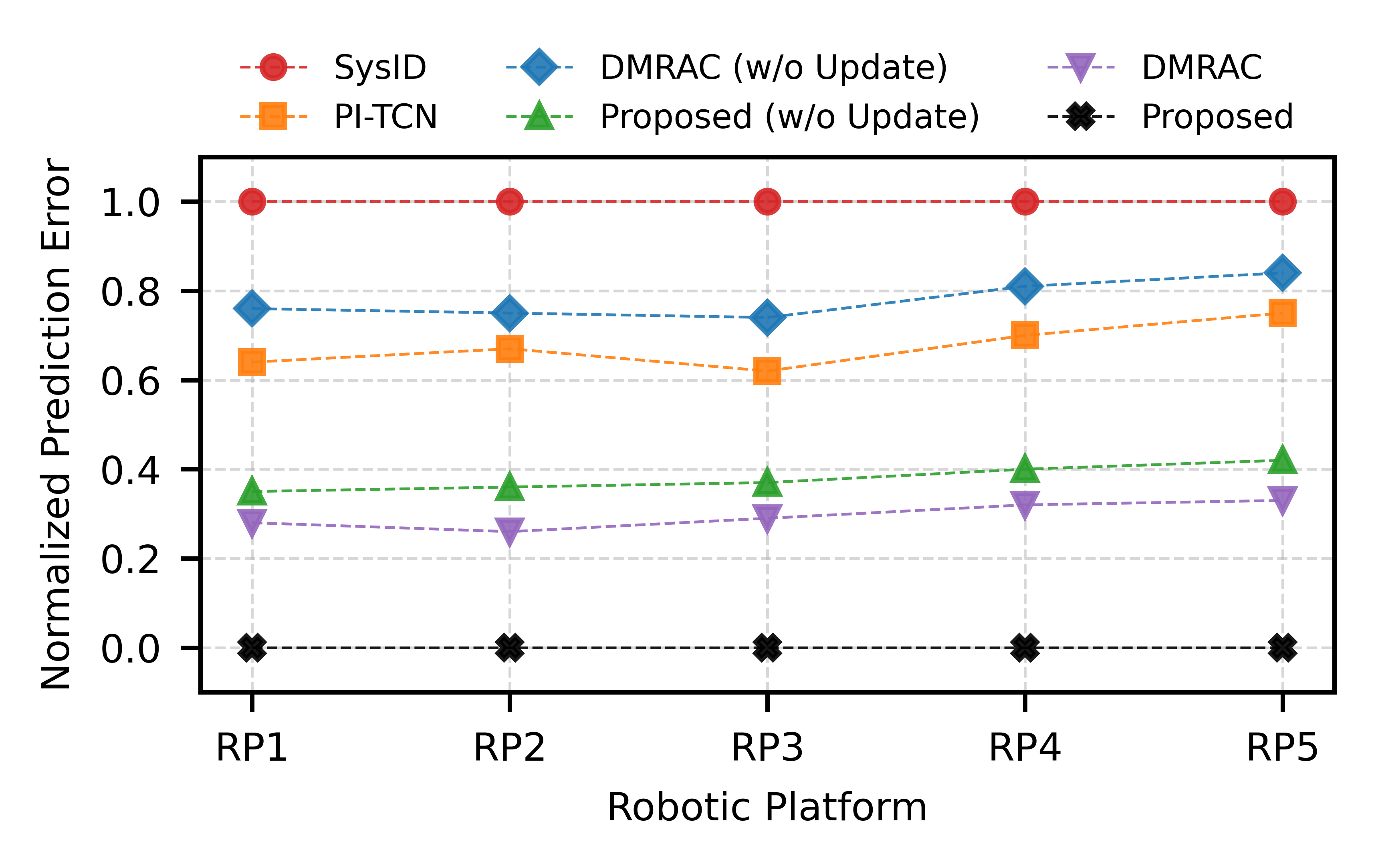}
    \caption{Normalized model prediction error across robotic platforms. Lower is better. RP1--RP5 denote the mobile robot, manipulator, quadrotor, mobile manipulator, and aerial manipulator, respectively.}
    \label{fig:prediction_error}
\end{figure}

Fig.~\ref{fig:prediction_error} compares normalized prediction error\footnotemark across platforms. Sparse SysID performs worst because it relies on a fixed dictionary of basis functions, and its accuracy is fundamentally limited when the true residual dynamics are not sparse in the chosen function space . DMRAC and PI-TCN improve performance by learning residual dynamics, but both treat the mismatch as an unstructured function; PI-TCN benefits from temporal encoding and physics-inspired regularization, yielding better generalization, while DMRAC further improves through online adaptation to non-stationary disturbances . However, these methods exhibit larger errors on the mobile and aerial manipulators due to stronger configuration-dependent coupling, which cannot be captured by purely additive residual models. The proposed method achieves the lowest error by explicitly modeling structured corrections to inertia and velocity coupling, while restricting online adaptation to disturbance-like components. This demonstrates that neither sparsity nor physics-informed residual learning alone is sufficient—accurate modeling requires structure-preserving dynamics together with selective online adaptation for real-world deployment.

\footnotetext{$\mathrm{RMS} = \sqrt{\frac{1}{T} \sum_{t=1}^{T} \|\delta_e(t)\|_2^2}$. RMS is normalized per platform using min--max normalization: $\mathrm{RMS}_{\mathrm{norm}} = \frac{\mathrm{RMS} - \mathrm{RMS}_{\min}}{\mathrm{RMS}_{\max} - \mathrm{RMS}_{\min}}$.}

\subsection{Trajectory Tracking}
We evaluate trajectory tracking performance for the scenarios shown in Fig.~\ref{fig:experiment_scenerio} across all robotic platforms. Each system is controlled using the law in~\eqref{eq:control_law}, and performance is quantified by the root-mean-square error (RMSE) of the tracking signal $e(t) = \chi(t) - \chi_d(t)$. 

\begin{figure}[!h]
\centering
\includegraphics[width=0.95\linewidth]{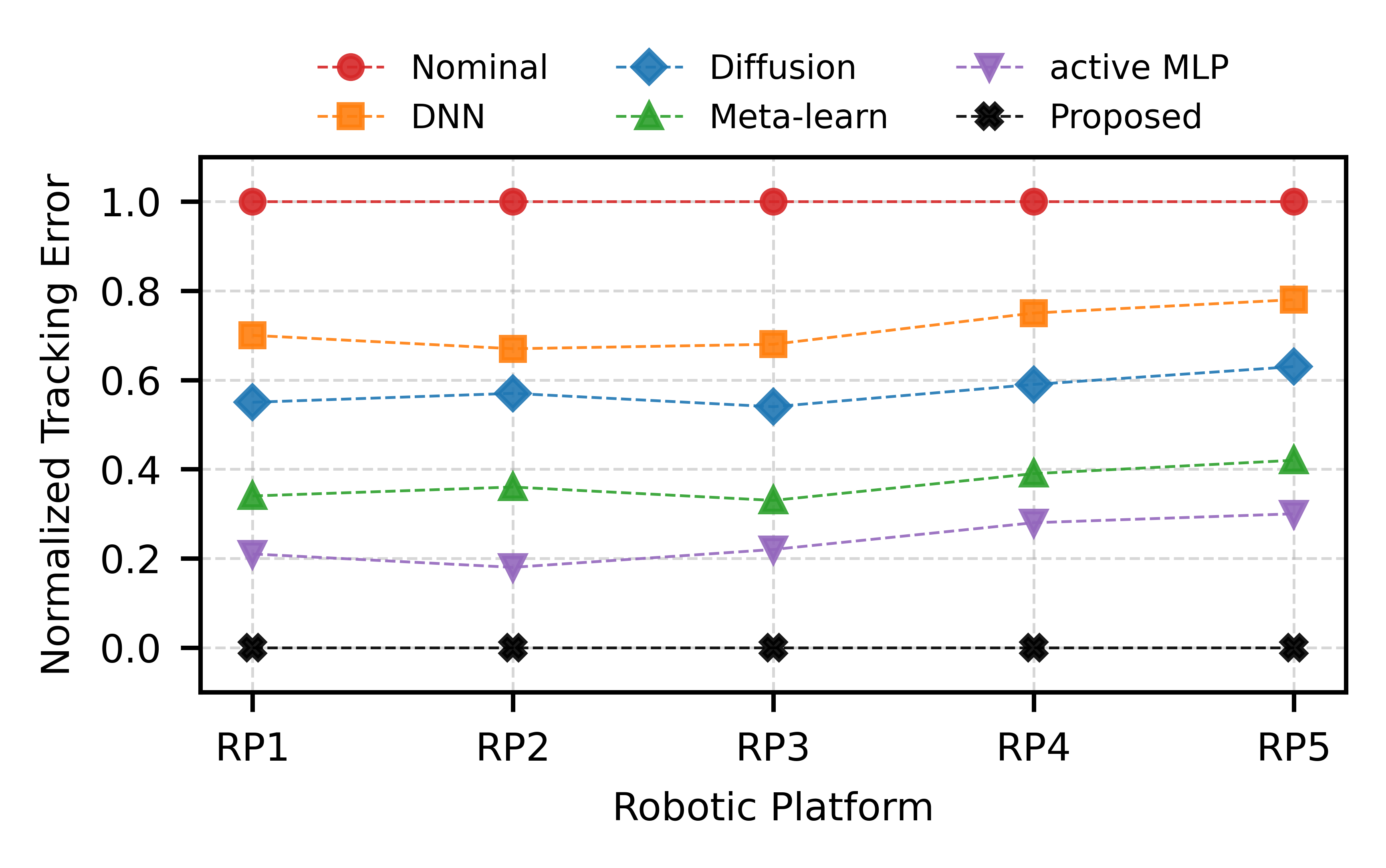}
\caption{Normalized trajectory tracking error across robotic platforms. Lower is better. RP1--RP5 denote the mobile robot, manipulator, quadrotor, mobile manipulator, and aerial manipulator, respectively.}
    \label{fig:tracking_error}
\end{figure}

Fig.~\ref{fig:tracking_error} reports normalized tracking error\footnotemark across platforms. The nominal model performs worst due to its inability to capture unmodeled aerodynamic effects and configuration changes, leading to systematic tracking degradation under real-world disturbances . Purely offline methods such as DNN and diffusion improve performance by learning residual dynamics, with diffusion outperforming DNN because it captures the multimodal and stochastic nature of real-world dynamics and mitigates temporal error propagation , yet both remain limited by distribution shift since they lack online adaptation . Meta-learning further improves tracking by enabling rapid adaptation through shared representations, but its adaptation is restricted to a low-dimensional span of pre-trained basis functions, limiting extrapolation to complex, time-varying disturbances . Active MLP achieves stronger performance by incorporating online updates, but its reliance on batch-based SGD introduces instability, sensitivity to noise, and inconsistent adaptation in fast, nonlinear regimes . In contrast, the proposed method achieves the best performance by combining structured dynamics learning with recursive, uncertainty-aware updates (akin to Bayesian/RLS), enabling stable, sample-efficient adaptation to non-stationary dynamics while preserving physical consistency—demonstrating that robust trajectory tracking requires both structured modeling and principled online adaptation.

\begin{figure*}[]
\centering
\vspace{-2mm}
\includegraphics[width=0.99\linewidth]{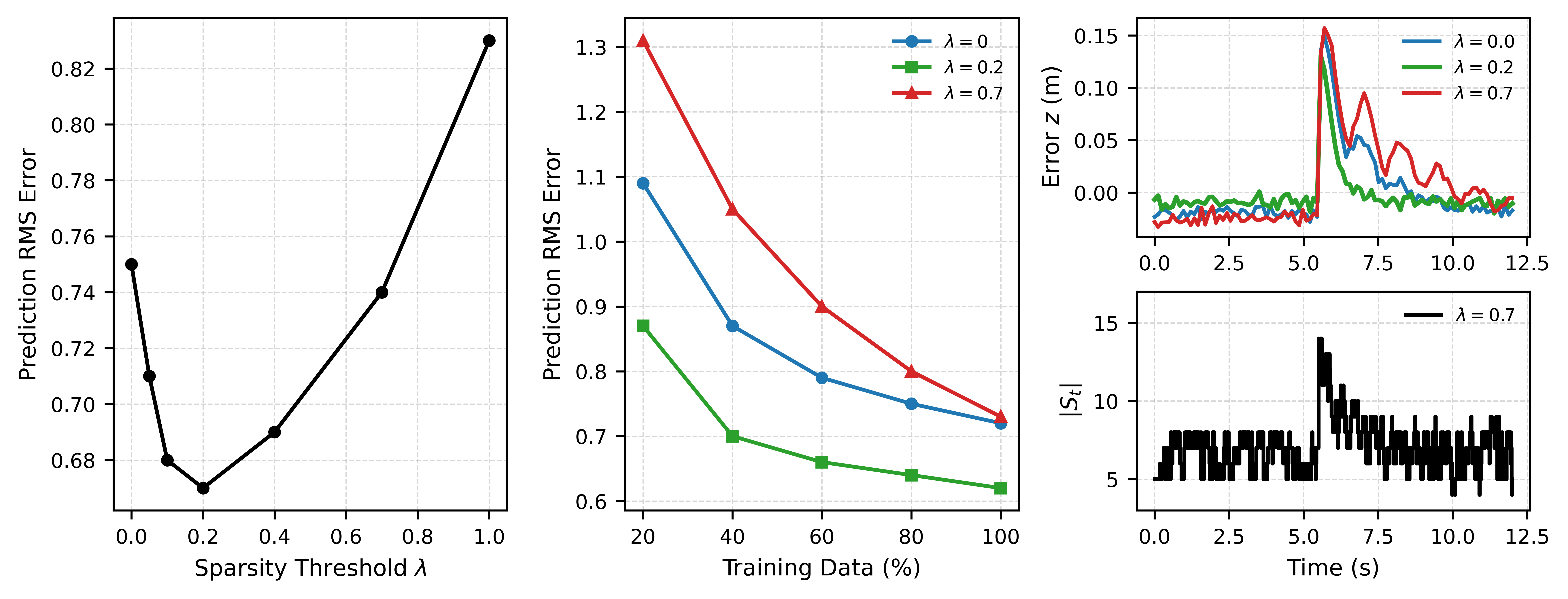}
\caption{Ablation study on the aerial manipulator illustrating the role of sparsity and adaptation. \textbf{Left:} Prediction error versus sparsity threshold $\lambda$, showing a trade-off between dense and overly sparse representations. \textbf{Middle:} Prediction error versus training data for different sparsity levels, highlighting improved data efficiency under moderate sparsity. \textbf{Top-right:} Tracking error along the $z$-axis under a payload drop disturbance, where sparse models exhibit faster and smoother recovery. \textbf{Bottom-right:} Evolution of the active latent dimension $|S_t|$ during the same event, showing increased mode activation during transient dynamics followed by a return to nominal levels.}
    \label{fig:ablation}
\end{figure*}

\subsection{Ablation Study: Effect of Sparsity and Adaptation}

We analyze the effect of sparsity on prediction accuracy, data efficiency, and adaptation using the aerial manipulator, where strong coupling dynamics are present.

Fig.~\ref{fig:ablation} (left) shows prediction error as a function of sparsity threshold $\lambda$. Dense models tend to over-parameterize residual dynamics, while moderate sparsity improves performance by suppressing spurious interactions. Excessive sparsity removes relevant dynamics, leading to underfitting. This confirms that an intermediate sparsity level provides the best balance between expressiveness and robustness.
Fig.~\ref{fig:ablation} (middle) evaluates prediction error versus training data size. Sparse models achieve consistently lower error in the low-data regime, indicating improved sample efficiency. By reducing the effective model dimension, sparsity enables faster convergence and better generalization, whereas dense models require more data and overly sparse models lack sufficient flexibility.

Fig.~\ref{fig:ablation} (top-right) shows prediction error following a payload release event, which induces a transient change in system dynamics. The sparse model exhibits faster error reduction and smoother convergence, highlighting the benefit of adapting within a reduced latent subspace.
Fig.~\ref{fig:ablation} (top-bottom) illustrates the evolution of the active latent dimension $|S_t|$. During nominal operation, only a small number of modes are active. Following the disturbance, additional modes are activated to capture transient dynamics, after which the model returns to a lower-dimensional representation. This demonstrates that the model dynamically adjusts its complexity based on operating conditions.

These results show that sparsity improves data efficiency and adaptation while preserving sufficient expressiveness, making it particularly effective for systems with strong dynamic coupling.

\section{Conclusion}
In this work, we presented a physics-aware residual learning framework for Euler--Lagrange robotic systems that combines structure-preserving mechanical corrections, sparse history-dependent interaction modeling, and selective online adaptation. By decomposing residual dynamics into inertia, induced Coriolis, and generalized-force components, the proposed model preserves key mechanical structure while adapting only the disturbance-sensitive part of the dynamics. This leads to improved predictive accuracy and trajectory tracking across multiple robotic platforms operating under coupled and changing conditions. The results also show that compact sparse representations can improve both data efficiency and online adaptation by restricting updates to a small set of active interaction patterns. A current limitation is that the mechanically structured component is kept fixed during deployment, which may be restrictive under a very large or persistent changes in inertial properties. Future work will therefore investigate joint adaptation of structured mechanical terms, richer uncertainty-aware updates, and extension to contact-rich and higher-dimensional manipulation tasks.

\bibliographystyle{IEEEtran}
\bibliography{our_bib}

\appendix

\begin{assumption}[Smoothness]
\label{ass:smooth}
The nominal dynamics $\bar M(\chi)$, $\bar C(\chi,\dot{\chi})$, and $\bar f(\chi,\dot{\chi})$ are continuously differentiable. $\bar M(\chi)$ is assumed to be symmetric positive definite. The learned correction $\Delta M_{\theta_M}(\chi, z^M)$ is continuously differentiable in $\chi$.
\end{assumption}

\begin{assumption}[Bounded Latent Variables]
\label{ass:bounded_latent}
The latent variables are bounded for all $t$, i.e., $
\|z_t^M\| \le c_M, \quad \|z_t^f\| \le c_f$,
for some constants $c_M, c_f > 0$.
\end{assumption}

\begin{assumption}[Bounded Approximation Error]
\label{ass:error}
The residual modeling error satisfies $
\|\delta_t - \hat{\delta}_t\| \le \epsilon_\delta$,
for some $\epsilon_\delta > 0$.
\end{assumption}

\begin{assumption}[Persistence of Excitation]
\label{ass:pe}
The masked latent features satisfy $
\sum_{t=t_0}^{t_0+T} \tilde z_t^f (\tilde z_t^f)^\top \succeq \beta I$,
for some $T > 0$ and $\beta > 0$.
\end{assumption}

\begin{lemma}[Positive Definiteness of Inertia]
The learned inertia matrix $\hat M(\chi, z^M)$ defined in \eqref{eq:corrected_inertia_general}--\eqref{eq:learned_deltaM} is symmetric positive definite for all $(\chi, z^M)$.
\end{lemma}

\begin{proof}
By construction, $ \hat M(\chi, z^M) = \bar M(\chi) + B(\chi, z^M) B(\chi, z^M)^\top + D(\chi, z^M)$, where $\bar M(\chi) \in \mathbb{R}^{n \times n}$ is symmetric positive definite by assumption, $B B^\top$ is symmetric positive semidefinite, and $D = \mathrm{diag}(\mathrm{softplus}(d))$ is diagonal with strictly positive entries.
Therefore, for any nonzero vector $x \in \mathbb{R}^n$, $ x^\top \hat M x = x^\top \bar M x + x^\top B B^\top x + x^\top D x > 0$, since $x^\top \bar M x > 0$ and $x^\top D x > 0$. Hence, $\hat M(\chi, z^M) \succ 0$ for all $(\chi, z^M)$.
\end{proof}

\begin{lemma}[Approximate Euler--Lagrange Structure Preservation]
The learned matrices $\hat M(\chi,z^M)$ and $\hat C(\chi,\dot{\chi},z^M)$ satisfy 
$\dot{\hat M} - 2\hat C = S(\chi,\dot{\chi},z^M)$,
where $S$ is skew-symmetric up to bounded approximation error.
\end{lemma}

\begin{proof}
The Coriolis matrix is constructed from $\hat M$ via Christoffel symbols. Under exact differentiation, the identity holds for Euler--Lagrange systems. In practice, automatic differentiation and function approximation introduce bounded deviations under Assumption~\ref{ass:smooth}, yielding approximate skew-symmetry. This ensures that the learned model respects the intrinsic coupling between inertia and velocity-dependent effects.
\end{proof}

\begin{proposition}[Bounded Parameter Estimates]
\label{prop:bounded_param}
Under Assumptions~\ref{ass:bounded_latent} and~\ref{ass:pe}, the parameter estimates $\mu_{t,j}$ and covariance matrices $P_{t,j}$ remain bounded.
\end{proposition}

\begin{proof}
The update corresponds to recursive Bayesian linear regression with bounded regressors $\tilde z_t^f$. Under persistence of excitation, the covariance matrix remains positive definite and bounded, which ensures bounded parameter estimates. This property guarantees stable online adaptation without parameter drift.
\end{proof}

\begin{proposition}[Effective Adaptation Dimension]
\label{prop:sparsity}
Let $z_t^f \in \mathbb{R}^{d_z}$ be the sparse latent code obtained via soft-thresholding, and define the active index set $
S_t = \{\, i \mid |(z_t^f)_i| > 0 \,\}, \quad |S_t| = s \ll d_z$.
Then the regression model $r_t = \Theta z_t^f$ depends only on the coordinates indexed by $S_t$, implying that adaptation occurs effectively in an $s$-dimensional subspace.
\end{proposition}

\begin{proof}
Since $(z_t^f)_i = 0$ for all $i \notin S_t$, the contribution of inactive coordinates to $r_t = \Theta z_t^f$ is zero. Hence, only the columns of $\Theta$ corresponding to indices in $S_t$ influence the prediction and parameter updates, restricting adaptation to those directions.
\end{proof}

\begin{proposition}[Error Reduction in Active Subspace]
Under persistent excitation of the active latent features, the prediction error of the generalized force component tends to decrease over time within the active subspace.
\end{proposition}

\begin{proof}
The result follows from standard convergence properties of recursive least squares under persistent excitation, restricted to the active latent coordinates.
\end{proof}

\begin{theorem}[Tracking Error Boundedness]
Consider the closed-loop system under the control law in~\eqref{eq:control_law}. Under Assumptions~\ref{ass:smooth}--\ref{ass:pe}, the tracking error $(e, \dot e)$ is uniformly ultimately bounded (UUB), i.e., $
\|(e(t), \dot e(t))\| \le c_1 e^{-c_2 t} + c_3 \epsilon_\delta$,
for some constants $c_1, c_2, c_3 > 0$.
\end{theorem}

\begin{proof}
Define the state vector $x = \begin{bmatrix} e \\ \dot e \end{bmatrix}$. The closed-loop error dynamics can be written as $\dot{x} = A x + B \sigma_{\delta}$, where {\footnotesize  $A = \begin{bmatrix} 0 & I \\ - K_p & - K_d \end{bmatrix}$ and $B = \begin{bmatrix} 0 \\ I \end{bmatrix}$ }, and $\sigma_{\delta}$ represents the bounded modeling error.
Since $K_p$ and $K_d$ are positive definite, the matrix $A$ is Hurwitz. Therefore, for any positive definite matrix $Q$, there exists a unique positive definite matrix $P$ satisfying the Lyapunov equation $A^\top P + P A = -Q$.

Consider the Lyapunov function $V = x^\top P x$. Its time derivative is $\dot V = x^\top (A^\top P + P A)x + 2 x^\top P B \sigma_{\delta}$. Using the Lyapunov equation, $\dot V = -x^\top Q x + 2 x^\top P B \sigma_{\delta}$.
Applying Cauchy--Schwarz and norm bounds, $\dot V \le -\lambda_{\min}(Q)\|x\|^2 + 2 \|P B\| \|x\| \|\sigma_{\delta}\|$.
Using Young’s inequality, $\dot V \le -\alpha \|x\|^2 + \beta \|\sigma_{\delta}\|^2$, for some constants $\alpha, \beta > 0$.

Since $\|\sigma_{\delta}\| \le \epsilon_\delta$, standard Lyapunov arguments imply that $x$ is uniformly ultimately bounded, yielding $\|x(t)\| \le c_1 e^{-c_2 t} + c_3 \epsilon_\delta$.
\end{proof}

\begin{corollary}
If the residual modeling error satisfies $\epsilon_\delta = 0$, the tracking error converges exponentially to zero.
\end{corollary}

\end{document}